\documentclass{llncs}

\usepackage[active]{srcltx}
\usepackage{amsfonts}
\usepackage{amsmath}
\usepackage{amssymb}
\usepackage{latexsym}
\usepackage{color}
\usepackage{graphics}
\usepackage{enumerate}
\usepackage{amstext}
\usepackage{psfig}
\usepackage{url}
\usepackage{bbm}
\usepackage{psfrag}

\def\Rset{\mathbb{R}}

\def\1{{\bf 1}}

\DeclareMathOperator*{\E}{\rm E}

\DeclareMathOperator*{\argmin}{\rm argmin}
\DeclareMathOperator{\reg}{\rm R}

\newcommand{\set}[1]{\left\{ #1 \right\}}

\newcommand{\lpref}{L}
\newcommand{\loss}{L_\omega}

\newcommand{\QS}[3]{Q_{#1}^{#3}(#2)}

\newcommand{\rnk}[2]{{#1}(#2)}
\newcommand{\ignore}[1]{}
 
\newtheorem{observation}[theorem]{Observation}

\begin{document} 

\title{ 
An Efficient Reduction of Ranking to Classification}

\titlerunning{An Efficient Reduction of Ranking to Classification}

\author{Nir Ailon\inst{1}
\and 
Mehryar Mohri\inst{2, 1}}

\institute{
Google Research,\\
76 Ninth Avenue, New York, NY 10011,\\
\email{nailon@google.com}.
\and
Courant Institute of Mathematical Sciences,\\
251 Mercer Street, New York, NY 10012,\\
\email{mohri@cims.nyu.edu}.
}

\maketitle

\begin{abstract}
This paper describes an efficient reduction of the learning problem of
ranking to binary classification.   The reduction guarantees an average
pairwise misranking regret of at most that of the binary classifier regret,
improving a recent result of Balcan et al which only guarantees a factor of $2$.
Moreover, our reduction applies to a
broader class of ranking loss functions, admits a simpler proof, and
the expected running time complexity of our algorithm in terms of
number of calls to a classifier or preference function is improved
from $\Omega(n^2)$ to $O(n \log n)$.  In addition, when the top $k$
ranked elements only are required ($k \ll n$), as in many applications
in information extraction or search engines, the time complexity of
our algorithm can be further reduced to $O(k \log k + n)$. Our
reduction and algorithm are thus practical for realistic applications
where the number of points to rank exceeds several thousands. Much of
our results also extend beyond the bipartite case previously studied.

Our rediction is a randomized one.  To complement our result, we 
also derive lower bounds on any deterministic reduction from 
binary (preference) classification to ranking, implying that our use
of a randomized reduction is essentially necessary for the guarantees
we provide.  
\end{abstract}

\section{Introduction}

The learning problem of ranking arises in many modern applications,
including the design of search engines, information extraction, and
movie recommendation systems. In these applications, the ordering of
the documents or movies returned is a critical aspect of the system.

The problem has been formulated within two distinct settings. In the
\emph{score-based setting}, the learning algorithm receives a labeled
sample of pairwise preferences and returns a \emph{scoring function}
$f\!:\!  U \rightarrow \Rset$ which induces a linear ordering of the
points in the set $U$. Test points are simply ranked according to the
values of $h$ for those points. Several ranking algorithms, including
RankBoost
\cite{DBLP:journals/jmlr/FreundISS03,DBLP:conf/colt/RudinCMS05},
SVM-type ranking \cite{Joachims02}, and other algorithms such as PRank
\cite{DBLP:conf/nips/CrammerS01,DBLP:conf/colt/AgarwalN05}, were
designed for this setting. Generalization bounds have been given in
this setting for the pairwise misranking error
\cite{DBLP:journals/jmlr/FreundISS03,DBLP:journals/jmlr/AgarwalGHHR05},
including margin-based bounds
\cite{DBLP:conf/colt/RudinCMS05}. Stability-based generalization
bounds have also been given in this setting for wide classes of
ranking algorithms both in the case of bipartite ranking
\cite{DBLP:conf/colt/AgarwalN05} and the general case \cite{icml,wea}.

A somewhat different two-stage scenario was considered in other
publications starting with Cohen, Schapire, and Singer
\cite{DBLP:journals/jair/CohenSS99}, and later Balcan et al.
\cite{conf/colt/BalcanBBCLS07}, which we will refer to as the
\emph{preference-based setting}. In the first stage of that setting, a
\emph{preference function} $h: U \times U \mapsto [0, 1]$ is learned,
where values of $h(u, v)$ closer to one indicate that $v$ is ranked
above $u$ and values closer to zero the opposite. $h$ is typically
assumed to be the output of a classification algorithm trained on a
sample of labeled pairs, and can be for example a convex combination
of simpler preference functions as in
\cite{DBLP:journals/jair/CohenSS99}. A crucial difference with the
score-based setting is that, in general, the preference function $h$
does not induce a linear ordering. The order it induces may be
non-transitive, thus we may have for example $h(u, v) = h(v, w) = h(w,
u) = 1$ for three distinct points $u$, $v$, and $w$. To rank a test
subset $V \subset U$, in the second stage, the algorithm orders the
points in $V$ by making use of the preference function $h$ learned in
the first stage.

This paper deals with the preference-based ranking setting just
described. The advantage of this setting is that the learning
algorithm is not required to return a linear ordering of all points in
$U$, which is impossible to achieve faultlessly in accordance with a
true pairwise preference labeling that is non-transitive. This is more
likely to be achievable exactly or with a better approximation when
the algorithm is requested instead, as in this setting, to supply a
linear ordering, only for a limited subset $V \subset U$.

When the preference function is learned by a binary classification
algorithm, the preference-based setting can be viewed as a reduction
of the ranking problem to a classification one. The second stage
specifies how the ranking is obtained using the preference function.

Cohen, Schapire, and Singer \cite{DBLP:journals/jair/CohenSS99} showed
that in the second stage of the preference-based setting, the general
problem of finding a linear ordering with as few pairwise misrankings
as possible with respect to the preference function $h$ is
NP-complete. The authors presented a greedy algorithm based on the
tournament \emph{degree} for each element $u \in V$ defined as the
difference between the number of elements $u$ is preferred to versus
the number of those preferred to $u$. The bound proven by these
authors, formulated in terms of the pairwise disagreement loss $l$
with respect to the preference function $h$, can be written as
$l(\sigma_{greedy}, h) \leq 1/2 + l(\sigma_{optimal}, h)/2$, where
$l(\sigma_{greedy}, h)$ is the loss achieved by the permutation
$\sigma_{greedy}$ returned by their algorithm and $l(\sigma_{optimal},
h)$ the one achieved by the optimal permutation $\sigma_{optimal}$
with respect to the preference function $h$.  This bound was given for
the general case of ranking, but in the particular case of bipartite
ranking (which we define below), a random ordering can achieve a pairwise disagreement loss of
$1/2$ and thus the bound is not informative.

More recently, Balcan et al \cite{conf/colt/BalcanBBCLS07} studied the
bipartite ranking problem and showed that sorting the elements of $V$
according to the same tournament degree used by
\cite{DBLP:journals/jair/CohenSS99} guarantees a pairwise misranking
regret of at most $2r$ using a binary classifier with regret $r$.
However, due to the quadratic nature of the definition of the
tournament degree, their algorithm requires $\Omega(n^2)$ calls to the
preference function $h$, where $n = |V|$ is the number of objects to
rank.

We describe an efficient algorithm for the second stage of
preference-based setting and thus for reducing the learning problem of
ranking to binary classification. We improve on the recent result of Balcan
et al. \cite{conf/colt/BalcanBBCLS07}, by guaranteeing an
average pairwise misranking regret of at most $r$ using a binary
classifier with regret $r$. In other words, we improve their
constant from $2$ to $1$.  Our reduction applies (with different constants) to a
broader class of ranking loss functions, admits a simpler proof, and
the expected running time complexity of our algorithm in terms of
number of calls to a classifier or preference function is improved
from $\Omega(n^2)$ to $O(n \log n)$.  Furthermore, when the top $k$
ranked elements only are required ($k \ll n$), as in many applications
in information extraction or search engines, the time complexity of
our algorithm can be further reduced to $O(k \log k + n)$. Our
reduction and algorithm are thus practical for realistic applications
where the number of points to rank exceeds several thousands. Much of
our results also extend beyond the bipartite case previously studied
by \cite{conf/colt/BalcanBBCLS07} to the general case of ranking. A
by-product of our proofs is also a bound on the pairwise disagreement
loss with respect to the preference function $h$ that we will compare
to the result given by Cohen, Schapire, and Singer
\cite{DBLP:journals/jair/CohenSS99}.

The algorithm used by Balcan et
al. \cite{DBLP:conf/icml/BeygelzimerDHLZ05} to produce a ranking based
on the preference function is known as sort-by-degree and has been
recently used in the context of minimizing the feedback arcset in
tournaments \cite{CFR06}. Here, we use a different algorithm,
QuickSort, which has also been recently used for minimizing the
feedback arcset in tournaments
\cite{DBLP:conf/stoc/AilonCN05,Ailon07}. The techniques presented make
use of the earlier work by Ailon et al. on combinatorial optimization
problems over rankings and clustering
\cite{DBLP:conf/stoc/AilonCN05,Ailon07}.

The remainder of the paper is structured as follows. In
Section~\ref{sec:set-up}, we introduce the definitions and notation
used in future sections and introduce a family of general loss
functions that can be used to measure the quality of a ranking
hypothesis. Section~\ref{sec:algorithm} describes a simple and
efficient algorithm for reducing ranking to binary classification,
proves several bounds guaranteeing the quality of the ranking produced
by the algorithm, and shows the running-time complexity of our
algorithm to be very efficient. In Section~\ref{sec:discussion} we
discuss the relationship of the algorithm and its proof with previous
related work in combinatorial optimization.  In Section~\ref{sec:lower}
we derive a lower bound of factor $2$ on any deterministic reduction from 
binary (preference) classification to ranking, implying that our use
of a randomized reduction is essentially necessary for the improved guarantees
we provide.

\section{Preliminaries}
\label{sec:set-up}

This section introduces several preliminary definitions necessary for
the presentation of our results.   In what follows, $U$ will denote
a universe of elements (e.g. the collection of all possible query-result
pairs returned by a web search task) and $V\subseteq U$ will denote
a small subset thereof (e.g. a preliminary list of relevant results
for a given query).   For simplicity of notation we will assume that $U$
is a set of integers, so that we are always able to choose a minimal (canonical)
element in a finite subset (as we do in (\ref{eq:opth}) below).  This
arbitrary ordering should not be confused with the ranking problem we are considering.

\subsection{General Definitions and Notation}

We first briefly discuss the learning setting and assumptions made by
Balcan et al.'s \cite{DBLP:conf/icml/BeygelzimerDHLZ05} and Cohen et
al.  \cite{DBLP:journals/jair/CohenSS99} and introduce a consistent
notation to make it easier to compare our results with that of this
previous work.

\subsubsection{Ground truth} 

In \cite{DBLP:conf/icml/BeygelzimerDHLZ05}, the \emph{ground truth} is
a \emph{bipartite ranking} of the set $V$ of elements that one wishes
to rank.\footnote{More generally, the ground truth may also be a
distribution of bipartite rankings, but the error bounds both in our
work and that of previous work are achieved by fixing one ground truth
and taking conditional expectation as a final step. Thus, we can
assume that it is fixed.} A bipartite ranking is a partition of $V$
into \emph{positive} and \emph{negative} elements where positive
elements are preferred over negative ones and elements sharing the
same label are in a tie.  This is a natural setting when the human
raters assign a positive or negative label to each element. Here, we
will allow a more general structure where the ground truth is a
ranking $\sigma^*$ equipped with a weight function $\omega$, which can
be used for encoding ties. The bipartite case can be encoded by
choosing a specific $\omega$ as we shall further discuss below.  

In \cite{DBLP:journals/jair/CohenSS99}, the "ground truth" has a different
interpretation, which we briefly discuss in Section~\ref{sec:lossbounds}.

\subsubsection{Preference function} 

In both \cite{DBLP:journals/jair/CohenSS99} and
\cite{DBLP:conf/icml/BeygelzimerDHLZ05}, a preference function $h:
U\times U\rightarrow [0,1]$ is assumed, which is learned in a first
learning stage.  The convention is that the higher $h(u, v)$ is, the
more our belief that $u$ should be ahead of $v$.  The function $h$
satisfies \emph{pairwise consistency}: $h(u, v) + h(v, u) = 1$, but
need not even be transitive on $3$-tuples.  The second stage uses $h$
to output a proper ranking $\sigma$, as we shall further discuss
below. The running time complexity of the second stage is measured
with respect to the number of calls to $h$.

\subsubsection{Output of learning algorithm} 

The final output of the second stage of the algorithm, $\sigma$, is a
proper ranking of $V$.  Its cost is measured differently in
\cite{DBLP:conf/icml/BeygelzimerDHLZ05} and
\cite{DBLP:journals/jair/CohenSS99}. In
\cite{DBLP:conf/icml/BeygelzimerDHLZ05}, it is measured against
$\sigma^*$ and compared to the cost of $h$ against $\sigma^*$. This
can be thought of as the best one could hope for if $h$ encodes all
the available information.  In \cite{DBLP:journals/jair/CohenSS99},
$\sigma$ is measured against the given preference function $h$, 
and compared to the best one can get.  

\ignore{
I think we do not want to emphasize this view. Rather, we want to 
present this as a learning problem. That is why I commented this out.

This setting is akin to an optimization problem because the cost is
measured with respect to a known object and not an unknown ground
truth or label.  The idea there is that the learning cost of $\sigma$
measured against $h^*$ can be, as a final step, bounded by the sum of
the costs of $\sigma$ measured against $h$ and $h$ measured against
$h^*$.  }

\subsection{Loss Functions}
We are now ready to define the loss functions used to measure the
quality of an output ranking $\sigma$ either with respect to
$\sigma^*$ (as in \cite{DBLP:conf/icml/BeygelzimerDHLZ05}) or
with respect to $h$ (as in \cite{DBLP:journals/jair/CohenSS99}).

Let $V \subseteq U$ be a finite subset that we wish to rank and let
$S(V)$ denote the set of \emph{rankings} on $V$, that is the set of
injections from $V$ to $[n] = \set{1, \dots, n}$, where $n = |V|$.  If
$\sigma \in S(V)$ is such a ranking, then $\sigma(u)$ is the rank of
an element $u\in V$, where "lower" is interpreted as "ahead".  More
precisely, we say that $u$ is preferred over $v$ with respect to
$\sigma$ if $\sigma(u) < \sigma(v)$.  For compatibility with the
notation used for general preference functions, we also write
$\sigma(u,v)=1$ if $\sigma(u)<\sigma(v)$ and $\sigma(u,v)=0$
otherwise.

The following general loss function $\loss$ measures the quality of a
ranking $\sigma$ with respect to a desired one $\sigma^*$ using a
weight function $\omega$ (described below):
\begin{equation}\label{eq:lossfuncdef}
\loss(\sigma, \sigma^*) = {n\choose 2}^{-1}\sum_{u\neq v} \sigma(u,v) \sigma^*(v,u)\omega(\rnk{\sigma^*}{u}, \rnk{\sigma^*}{v})\ .
\end{equation}
The sum is over all pairs $u,v$ in the domain $V$ of the rankings $\sigma, \sigma^*$.
It counts the number of inverted pairs $u,v\in V$ weighed by $\omega$, which assigns
importance coefficients to pairs, based on their positions in $\sigma^*$.
The function $\omega$ must satisfy the following three natural axioms, which
will be necessary in our analysis:
\begin{enumerate}
\item[\small (P1)\!\!] Symmetry: $\omega(i, j) = \omega(j, i)$ for
all $i,j$;

\item[\small (P2)\!\!] Monotonicity: $\omega(i, j) \leq \omega(i,
k)$ if either $i < j < k$ or $i > j > k$;

\item[\small (P3)\!\!] Triangle inequality: $\omega(i, j) \leq
\omega(i, k) + \omega(k, j)$.
\end{enumerate}
This definition is very general and encompasses many useful, well
studied distance functions.  Setting $\omega(i, j) = 1$ for all $i
\neq j$ yields the unweighted pairwise misranking measure or the
so-called Kemeny distance function. 

For a fixed integer $k$, the following function
\begin{equation}\label{topk}
 \omega(i,j) = 
\begin{cases} 1 & \mbox{if } ((i \leq k) \vee (j \leq k)) \wedge (i\neq j) \\
0 & \mbox{otherwise},
\end{cases} 
\end{equation}
can be used to emphasize ranking at the top $k$ elements. Misranking
of pairs with one element ranked among the top $k$ is penalized by
this function. This can be of interest in applications such as
information extraction or search engines where the ranking of the top
documents matters more. For this emphasis function, all elements
ranked below $k$ are in a tie. In fact, it is possible to encode any
tie relation using $\omega$.

The loss function considered in \cite{conf/colt/BalcanBBCLS07} can
also be straightforwardly encoded with the following emphasis \emph{bipartite} function
\begin{equation}\label{eq:bipartitedef}
\omega(i,j) = 
\begin{cases} 1 & (i\leq k) \wedge (j>k) \\
1 & (j\leq k) \wedge (i>k) \\
0 & \mbox{otherwise}.
\end{cases}
\end{equation}
Items in positions $1, \dots, k$ for the permutation $\sigma$ can be
thought of as the positive items (in a tie), and those in $k+1, \dots,
|V|$ as negative (also in a tie). 
This choice coincides with $ (1- \mbox{AUC})$, where AUC is the area under the ROC curve commonly
used in statistics and machine learning problems
\cite{hanley,lehmann}.

Clearly, setting $\omega(i,j) =
|s(i)-s(j)|$ for any monotone \emph{score} function $s$ works as well.  It is well known though that such a function can in fact be expressed as a
convex combination of functions of the type (\ref{eq:bipartitedef}).  Hence, a bipartite function should be thought of as the simplest 

In general, one may wish to work with a collection of ground truths
$\sigma^*_1, \dots, \sigma^*_N$ and weight functions $\omega_1,\dots,
\omega_N$ and a loss function which is a sum over the individual
losses with respect to $\sigma^*_i, \omega_i$, e.g. in meta
searching\footnote{The reason we have separate weight function
$\omega$'s is e.g. each search engine may output top-$k$ outputs for
different values of $k$.}. Since our bound is based on the expected
loss, it will straightforwardly generalize to this setting using the
linearity of expectation. Thus, we can assume without loss of
generality a single ground truth $\sigma^*$ equipped with a single
$\omega$.

\subsubsection{Preference Loss Function}

We need to extend the definition to measure the loss of a preference 
function $h$ with respect to $\sigma^*$.  Recall that $h(u,v)$ is
In contrast with the loss function just defined, we need to define a
\emph{preference loss} measuring a ranking's disagreements with
respect to a preference function $h$.  When measured against $\sigma^*, \omega$,
the function $\loss$ can be readily used:
\begin{equation}\label{eq:preflossdef}
\loss(h, \sigma^*) = {n\choose 2}^{-1} \sum_{u\neq v} h(u,v) \sigma^*(v,u)\omega(\rnk{\sigma^*}{u}, \rnk{\sigma^*}{v})\ .
\end{equation}
We use $\lpref$ to denote $\loss$ for the special case where $\omega$
is the constant function $1$, $\omega = 1$:
\begin{equation}
\lpref(h, \sigma^*) = {n\choose 2}^{-1}\sum_{u\neq v} h(u,v) \sigma(v,u)\ .
\end{equation}
The special case of $\lpref$ coincides with the standard pairwise disagreement
loss of a ranking with respect to a preference function as used in
\cite{DBLP:journals/jair/CohenSS99}.

\subsection{The Special Bipartite Case}
A particular case of interest is when $\omega$ belongs to the family
of weight functions defined in (\ref{eq:bipartitedef}).  For this particular case we
will use a slightly more convenient notation.  For a set of elements $V\subseteq U$, Let
$\Pi_V$ denote
the set of partitions of $V$ into two sets (positive and negative).
More formally, $\tau\in \Pi_V$ is a function from $V$ to $\{0,1\}$
(where $0$ should be thought of as the \emph{preferred} or \emph{positive}
value, and $1$ the negative; we choose this convention so that $\tau(u)$ can
be interepreted as the \emph{rank} of $u$, where there are two possible ranks).
Abusing notation we say that $\tau(u,v)=1$ if $\tau(u) < \tau(v)$ ($u$ is preferrede over $v$)
and $\tau(u,v)=0$ otherwise (note that here we can have $\tau(u,v)=\tau(v,u)=0$).
Our abuse of notation allows us to use the readily defined function $\lpref$ to measure the loss of a ranking $\sigma\in S_V$ against $\tau^* \in \Pi_V$
(which will usually take the role of a ground truth):
$$\lpref(\sigma, \tau^*) = {n\choose 2}^{-1}\sum_{u\neq v}{\sigma(u,v)\tau^*(v,u)}\ .$$
Note that this coincides with $L_{\omega}(\sigma, \sigma^*)$, where
$\sigma^*$ is any ranking on $V$ with $\sigma^*(u)<\sigma^*(v)$ 
whenever $\tau^*(u)<\tau^*(v)$, and $\omega$ is as in (\ref{eq:bipartitedef})
with $$k=|\{u\in V: \tau^*(u)=0\}|\ .$$

\emph{A note on normalization:} The bipartite case is the one considered in \cite{conf/colt/BalcanBBCLS07}, with a small different which
is crucial for some of the bounds we derive.
There, the loss function is defined as
\begin{equation}\label{normalization}{|\{u,v: \tau^*(u) < \tau^*(v)\}|}^{-1}\sum_{u\neq v}{\sigma(u,v)\tau^*(v,u)}\ .\end{equation}
If we are working with just one $\tau^*$, the two loss functions are the same up to a constant. 
However, if we have a distribution over $\tau^*$ and consider the expected loss, then there may be a difference.
For simplicity we will work with the definition derived from (\ref{eq:preflossdef}), and will leave the other choice 
for discussion in Section~\ref{sec:discussion}.

\subsection{Independence on Irrelevant Alternatives and Regret Functions}
The subset $V$ is chosen from the universe $U$ from some distribution.  Together with $V$, a ground
truth ranking $\sigma^*\in S(V)$, and an admissible weight function $\omega$ are also chosen randomly.
We let $D$ denote the distribution on $V, \sigma^*, \omega$.  (In the bipartite case, $D$ is a distribution on
$V$ and on $\tau^*\in \Pi_v$.)
\indent
\begin{definition}
A distribution $D$ on $V,\sigma^*, \omega$ satisfies the pairwise independence on irrelevant alternatives (IIA) property if
 for all
 distinct $u,v \in U$, conditioned on $u,v \in V$ the random variables $\sigma^*(u,v)\omega(u,v)$ and $V\setminus\{u,v\}$ are independent.
\end{definition}
In the bipartite case this translates to
\begin{definition}
A distribution $D$ on $V,\tau^*$ satisfies the pairwise IIA property if
for all
 distinct $u,v\in U$, conditioned on $u,v \in V$ the random variables $\tau^*(u,v)$ and $V\setminus\{u,v\}$ are independent.
\end{definition}

Note that in the bipartite case, $D$ can satisfy pairwise IIA
while not satisfying  pointwise IIA.  (Pointwise IIA means that conditioned on $u\in V$, $\tau^*(u)$
and $V\setminus\{u\}$ are independent.)
In certain applications, e.g. when ground truth is obtained from humans,
it is reasonable \emph{not} to assume pointwise IIA.
Think of the "grass is greener" phenomenon: a satisfactory
option may seem unsatisfactory in the presence of an alternative.  
Continuing the
analogue, assuming pairwise IIA means that choosing  between two options
does not depend on the presence of a third alternative.  (By \emph{choosing} we
mean that ties are allowed.)

In this work we do not assume pointwise IIA, and when deriving loss bounds
we will not assume pairwise IIA either.  We will need pairwise IIA when working
with \emph{regret}, which is an adjustment of the loss designed so 
that an optimal solution would have a value of $0$ with respect to the ground
truth.  As pointed out in \cite{conf/colt/BalcanBBCLS07}, the regret measures
the loss modulo "noise".   

Using regret (here) makes sense when the optimal solution
has a strictly positive loss value.  In our case it can only happen if the ground truth
 is a proper distribution, namely, the probability mass is not concentrated on one point.

To define ranking regret, assume we are learning how to obtain a full ranking $\sigma$ of $V$, using
an algorithm $A$, so that $\sigma = A_s(V)$, where $s$ is a random stream of bits
possibly used by the algorithm.
For ranking learning, we define the regret of $A$ against $D$ as
\begin{equation*}
\begin{split}
\reg_{rank}(A,D) &= E_{V,\sigma^*,\omega, s} [L_\omega(A_s(V), \sigma^*)] 
-\min_{\tilde \sigma \in S(U)} E_{V,\sigma^*,\omega }[L_\omega(\tilde \sigma_{|V},
\sigma^*)]\ ,
\end{split}
\end{equation*}
where $\tilde \sigma_{|V}\in S(V)$ is defined by restricting the ranking $\tilde \sigma\in S(U)$ to $V$
in a natural way.

In the preference classification setting, it makes sense to define the regret of a preference
function $h: U\times U \mapsto \{0,1\}$ as follows: 
\begin{equation*}
\reg_{class}(h,D) = E_{V,\sigma^*, \omega} [L_\omega(h_{|V},\sigma^*)] - \min_{\tilde h} E_{V,\sigma^*, \omega}[L_\omega(\tilde h_{|V}, \sigma^*)]\ ,
\end{equation*}
where the minimum is over $\tilde h$ a preference function over $U$, and $\cdot_{|V}$ is
a restriction operator on preference functions defined in the natural way.
For the bipartite special case, we have the simplified form: 
\begin{eqnarray}\label{eq:regretdef}
\reg_{rank}(A,D) &=& E_{V,\tau^*, s} [\lpref(A_s(V), \tau^*)] -
\min_{\tilde \sigma \in S(U)} E_{V,\tau^*}[\lpref(\tilde \sigma_{|V},
\tau^*)] \label{eq:regrankbidef}\\
\reg_{class}(h,D)& = &E_{V,\tau^*} [\lpref(h_{|V},\tau^*)] - \min_{\tilde h} E_{V,\tau^*}[\lpref(\tilde h_{|V}, \tau^*)]\label{eq:regclassbidef}\ .
\end{eqnarray}

The regret measures how well an algorithm or a classifier performs compared to the best "static" algorithm, namely, one that ranks $U$ in advance (in $\reg_{rank}$)
or provides preference information on $U$ in advanced (in $\reg_{class}$).  Note that the minimizer $\tilde h$ in (\ref{eq:regclassbidef}) can be easily found
by considering each $u,v\in U$ separately.  More precisely, one can take
\begin{equation}\label{eq:opth}
 \tilde h(u,v) = \begin{cases} 1 & E_{\tau^*}[\tau^*(u,v) | u,v \in V]> E_{\tau^*}[\tau^*(v,u)|u,v\in V] \\
0 & E_{\tau^*}[\tau^*(u,v) | u,v \in V]< E_{\tau^*}[\tau^*(v,u)|u,v\in V] \\
\1_{u>v} & \mbox{otherwise (equality)}
\end{cases} 
\end{equation}

Now notice that if $D$ satisfies pairwise IIA, then for any set $V_0$ containing $u,v$,
$$E_{\tau^*}[\tau^*(u,v)|V=V_0] = E_{\tau^*}[\tau^*(u,v) | u,v\in V]\ .$$
Therefore, in this case the $\min_{\tilde h}$ and $E_V$ operators commute:
$$ \min_{\tilde h} E_{V,\tau^*}[\lpref(\tilde h_{|V}, \tau^*)] = E_V\min_{\tilde h}E_{\tau^*}[\lpref(\tilde h_{|V}, \tau^*)]\ .$$
For our analysis it will indeed be useful to swap the $\min$ and $E_V$ operators.  We define
\begin{eqnarray}\label{eq:regretdef}
\reg'_{rank}(A,D) &=& E_{V,\tau^*, s} [\lpref(A_s(V), \tau^*)] -
E_V\min_{\tilde \sigma \in S(V)} E_{\tau^*}[\lpref(\tilde \sigma,
\tau^*)] \label{eq:regprankbidef}\\
\reg'_{class}(h,D)& = &E_{V,\tau^*} [\lpref(h_{|V},\tau^*)] - E_V\min_{\tilde h} E_{\tau^*}[\lpref(\tilde h, \tau^*)]\label{eq:regpclassbidef}\ ,
\end{eqnarray}
where now $\min_{\tilde h}$ is over preference functions $\tilde h$ on $V$.  We summarize this section with the following:
\begin{observation}\label{obs:regiia}
\begin{enumerate}
\item In general (using the concavity of $\min$ and Jensen's inequality): $ R'_{rank}(A,D) \geq R_{rank}(A,D)$;
\item Assuming pairwise IIA: $R'_{class}(h,D) = R_{class}(h,D)$\ .
\end{enumerate}
\end{observation}

\section{Algorithm for Ranking Using a Preference Function}
\label{sec:algorithm}

This section describes and analyzes an algorithm for obtaining a
global ranking of a subset using a prelearned preference function $h$,
which corresponds to the second stage of the preference-based setting.
Our bound on the loss will be derived using conditional expectation on
the preference loss assuming a fixed subset $V \subset U$, and fixed
$\sigma^*$ and $\omega$. To further simplify the analysis, we assume
that $h$ is binary, that is $h(u, v) \in \set{0, 1}$ for all $u, v \in
U$. 

\subsection{Description}

One simple idea to obtain a global ranking of the points in $V$
consists of using a standard comparison-based sorting algorithm where
the comparison operation is based on the preference function. However,
since in general the preference function is not transitive, the
property of the resulting permutation obtained is unclear.

This section shows however that the permutation generated by the
standard QuickSort algorithm provides excellent
guarantees.\footnote{We are not assuming here transitivity as in most
textbook presentations of QuickSort.} Thus, the algorithm we suggest
is the following. Pick a random \emph{pivot} element $u$ uniformly at
random from $V$. For each $v \neq u$, place $v$ on the
left\footnote{We will use the convention that ranked items are written
from left to right, starting with the most preferred ones.} of $u$ if
$h(v,u)=1$, 
and to its right otherwise. Proceed recursively with the array
to the left of $u$ and the one to its right and return the
concatenation of the permutation returned by the left recursion, $u$,
and the permutation returned by the right recursion.

We will denote by $\QS{s}{V}{h}$ the permutation resulting in running
QuickSort on $V$ using preference function $h$, where $s$ is the
random stream of bits used by QuickSort for the selection of the
pivots. As we shall see in the next two sections, on average, this
algorithm produces high-quality global rankings in a time-efficient
manner.

\subsection{Ranking Quality Guarantees}

The following theorems give bounds on the ranking quality of the
algorithm described, for both loss and regret, on the general and bipartite cases.

\begin{theorem}[Loss bounds in general case]
\label{th:approx3}
For any fixed subset $V\subseteq U$, preference function $h$ on $V$, ranking $\sigma^*\in S(V)$ and admissible weight function $\omega$ the
following bound holds:
\begin{equation}
\E_s[\loss(\QS{s}{V}{h}, \sigma^*)] \leq 2\loss(h, \sigma^*)
\end{equation}
\end{theorem}
Note: This implies by the principle of conditional expectation that
\begin{equation}
\E_{D,s}[\loss(\QS{s}{V}{h}, \sigma^*)] \leq 2E_D[\loss(h, \sigma^*)]
\end{equation}
(where $h$ can depend on $V$).

\begin{theorem}[Loss and regret bounds in bipartite case]
\label{th:approx3bipartite}
For any fixed $V\subset U$, preference function $h$ over $V$ and $\tau^*\in\Pi(V)$, the following bound holds:
\begin{equation}\label{eq:lossbipartitebound}
\begin{split}
\E_{s}[\lpref(\QS{s}{V}{h}, \tau^*] &= \lpref(h, \tau^*) \ .\\
\end{split}
\end{equation}
If $V,\tau^*$ are drawn from some distribution $D$ satisfying  pairwise IIA, then 
\begin{equation}
\begin{split}
\reg_{rank}(\QS{s}{\cdot}{h}, D) \leq \reg_{class}(h, D)
\end{split}
\end{equation}
\end{theorem}

 Note: Equation (\ref{eq:lossbipartitebound}) implies by the principle of conditional expectation that if $V,\tau^*$ are drawn from a
distribution $D$, then
\begin{equation}
\E_{D,s}[\lpref(\QS{s}{V}{h}, \tau^*)] = E_D[\lpref(h, \tau^*)]
\end{equation}
(where $h$ can depend on $V$).

To prove these theorems, we must first introduce some tools to help analyze QuickSort.  These tools were first developed in \cite{DBLP:conf/stoc/AilonCN05} in
the context of optimization, and here we initiate their use in learning.

\subsection{Analyzing QuickSort}
Assume $V$ is fixed, and let $Q_s = Q_s^h(V)$ be the (random) ranking outputted by QuickSort on $V$ using
preference function $h$.
During the execution of QuickSort, the order
between two points $u, v \in V$ is determined in one of two ways:

\begin{itemize}
\item Directly: $u$ (or $v$) was selected as the pivot with $v$
(resp. $u$) present in the same sub-array in a recursive call to
QuickSort. We denote by $p_{uv} = p_{vu}$ the probability of that
event.  In that case, the algorithm orders $u$ and $v$ according to
the preference function $h$.

\item Indirectly: a third element $w \in V$ is selected as pivot with
$w, u, v$ all present in the same sub-array in a recursive call to
QuickSort, $u$ is assigned to the left sub-array and $v$ to the right
(or vice-versa).

Let $p_{uvw}$ denote the probability of the event that $u$, $v$, and
$w$ be present in the same array in a recursive call to QuickSort and
that one of them be selected as pivot.  Note that conditioned on that
event, each of these three elements is equally likely to be selected
as a pivot since the pivot selection is based on a uniform
distribution.

If (say) $w$ is selected among the three, then $u$ will be placed on the
left of $v$ if $h(u,w)=h(w,v)=1$, and to its right if $h(v,w)=h(w,u)=1$.  In
all other cases, the order between $u,v$ will be determined only in a deeper nested
call to QuickSort.

\end{itemize}

Let $X,Y : V\times V\mapsto\Rset$ be any two functions on ordered pairs $u,v\in V$, and let $Z: {V \choose 2}\mapsto \Rset$ be a function on unordered pairs (sets of two elements). By convention, we use $X(u,v)$ to denote ordered arguments, and $Y_{uv}$ to denote unordered arguments.  We  define three functions $\alpha[X,Y] :{V\choose 2}\mapsto \Rset$, $\beta[X] : {V\choose 3}\mapsto \Rset$
and $\gamma[Z] : {V\choose 3}\mapsto \Rset$ as follows:
\begin{equation}
\begin{split}
\alpha[X,Y]_{uv} &= X(u,v)Y(v,u) + X(v,u)Y(u,v) \\
\beta[X]_{uvw} &= \frac 1 3 (h(u,v)h(v,w)X(w,u) + h(w,v)h(v,u)X(u,w)) \\
            &+ \frac 1 3 (h(v,u)h(u,w)X(w,v) + h(w,u)h(u,v)X(v,w)) \\
            &+ \frac 1 3 (h(u,w)h(w,v)X(v,u) + h(v,w)h(w,u)X(u,v)) \\
\gamma[Z]_{uvw} &= \frac 1 3 (h(u,v)h(v,w) + h(w,v)h(v,u))Z_{uw} \\
               &+ \frac 1 3 (h(v,u)h(u,w) + h(w,u)h(u,v))Z_{vw} \\
               &+ \frac 1 3 (h(u,w)h(w,v) + h(v,w)h(w,u))Z_{uv}\ . \\
\end{split}
\end{equation}

\begin{lemma}[QuickSort decomposition] \label{claim:decomp}
\begin{enumerate}
\item
For any $Z: {V\choose 2}\mapsto \Rset$, 
$$ \sum_{u<v} Z_{uv} = \sum_{u<v} p_{uv} Z_{uv} + \sum_{u<v<w} p_{uvw} \gamma[Z]_{uvw}\ .$$
 \item For any $X : V\times V\mapsto \Rset$,
$$ E_s[\sum_{u<v} \alpha[Q_s, X]_{uv}] = \sum_{u<v} p_{uv} \alpha[h, X]_{uv} + \sum_{u<v<w} p_{uvw} \beta[X]_{uvw}\ .$$
\end{enumerate}
\end{lemma}
\begin{proof}
To see the first part, notice that  for 
every unordered pair $u<v$ the expression $Z_{uv}$ is accounted for on the RHS of the equation
with total coefficient: $$p_{uv} + \sum_{w\not \in \{u,v\}} \frac 1 3 p_{uvw} (h(u,w)h(w,v) + h(v,w)h(w,u))\ .$$
Now, $p_{uv}$
is the probability that the pair $uv$ is charged directly (by definition), and $\frac 1 3 p_{uvw} (h(u,w)h(w,v) + h(v,w)h(w,u))$
is the probability that the pair $u,v$ is charged indirectly via $w$ as pivot.  Since each pair is charged exactly once,
these probabilities are of pairwise disjoint events that cover the probability space.  Hence, the total coefficient of $Z_{uv}$ on the
RHS is $1$, as is on the LHS.
The second part is proved similarly.
\end{proof}

\subsection{Loss Bounds}\label{sec:lossbounds}


We prove the first part of Theorems~\ref{th:approx3} and~\ref{th:approx3bipartite}.
We start with the general case notation. 
The loss incurred by QuickSort is (as a function of the random bits $s$), for fixed $\sigma^*, \omega$, clearly 
$\loss(Q_s, \sigma^*) = {n\choose 2}^{-1}\sum_{u<v} \alpha[Q_s, \Delta]_{uv}$, where $\Delta: V\times V\mapsto \Rset$ is defined as 
$ \Delta(u,v)  = \omega(\sigma^*(u), \sigma^*(v))\sigma^*(u,v)\ .$
By the second part of Lemma~\ref{claim:decomp}, the expected loss is therefore
\begin{equation}\label{eq:exploss}
 \E_s[\loss(Q_s, \sigma^*)] = {n\choose 2}^{-1}\left (\sum_{u<v} p_{uv} \alpha[h, \Delta]_{uv} + \sum_{u<v<w} p_{uvw}\beta[\Delta]_{uvw}\right )\ .
\end{equation}

Similarly, we have that $\loss(h,\sigma^*) = {n\choose 2}^{-1}\sum_{u<v} \alpha[h,\Delta]_{uv}$.
Therefore, using the first part of Lemma~\ref{claim:decomp},
\begin{equation}\label{eq:lossdecomp}
 \loss(h,\sigma^*) = {n\choose 2}^{-1}\left (\sum_{u<v} p_{uv} \alpha[h,\Delta]_{uv} + \sum_{u<v<w} \gamma[\alpha[h,\Delta]]_{uvw}\right )\ . 
\end{equation}



To complete the proof for the general (non-bipartite) case, it suffices to show that for all $u,v,w$, $\beta[\Delta]_{uvw} \leq 2\gamma[\alpha[h,\Delta]]_{uvw}$.
Up to symmetry, there are two cases to consider.  The first case assumes $h$  induces a cycle on $u,v,w$, and the second assumes it doesn't.
\begin{enumerate}
\item Without loss of generality, assume $h(u,v)=h(v,w)=h(w,u)$.  Plugging in the definitions, we get 
\begin{eqnarray}
\beta[\Delta]_{uvw} &=& \frac 1 3 (\Delta(u,v) + \Delta(v,w) + \Delta(w,u)) \label{eq:cyclegeneral1}\\
\gamma[\alpha[h,\Delta]]_{uvw} &=& {\frac 1 3} (\Delta(v,u) + \Delta(w,v) + \Delta(u,w)) \label{eq:cyclegeneral2}\ .
\end{eqnarray}
By the properties (P1)-(P3) of $\omega$, transitivity of $\sigma^*$ and definition of $\Delta$, we easily get that $\Delta$ satisfies the triangle inequality:
\begin{eqnarray*}
\Delta(u,v) &\leq& \Delta(u,w) + \Delta(w,v) \\
\Delta(v,w) &\leq& \Delta(v,u) + \Delta(u,w) \\
\Delta(w,u) &\leq& \Delta(w,v) + \Delta(v,u)\\
\end{eqnarray*}
Summing up the three equations, this implies that $\beta[\Delta]_{uvw} \leq 2\gamma[\alpha[h,\Delta]]_{uvw}$.
\item
 Without loss of generality, assume $h(u,v)=h(v,w)=h(u,w)=1$.  By plugging in the definitions, this implies that
\begin{eqnarray*}
\beta[\Delta]_{uvw} = \gamma[\alpha[h,\Delta]]_{uvw} = \alpha[h,\Delta]_{uw}\ ,
\end{eqnarray*}
as required.
\end{enumerate}
This concludes the proof for the general case.  As for the bipartite case, (\ref{eq:cyclegeneral1}-\ref{eq:cyclegeneral2}) translates to
\begin{eqnarray}
\beta[\Delta]_{uvw} &=& \frac 1 3 (\tau^*(u,v) + \tau^*(v,w) + \tau^*(w,u)) \label{eq:beta} \\
\gamma[\alpha[h,\Delta]]_{uvw} &=& {\frac 1 3} (\tau^*(v,u) + \tau^*(w,v) + \tau^*(u,w)) \label{eq:gamma} \ .
\end{eqnarray}
It is trivial to see that the two expressions are identical for any partition $\tau^*$ (indeed, they count the number of times we cross the partition from left to right when 
going in a circle on $u,v,w$: it does not matter in which direction we are going).  This concludes the loss bound part of Theorems~\ref{th:approx3} and~\ref{th:approx3bipartite}.

\qed
We place Theorem~\ref{th:approx3} in the framework used
by Cohen et al \cite{DBLP:journals/jair/CohenSS99}.  There, the objective is to find a ranking $\sigma$ that
has a low loss measured against $h$ compared to the \emph{theoretical} optimal ranking $\sigma_{optimal}$. 
Therefore, the problem considered there (modulo learning a preference function $h$) is a combinatorial optimization 
and not a learning problem. 
More precisely, we define
$$\sigma_{optimal} = \argmin_\sigma \lpref(h,\sigma)$$
and want to minimize $\lpref(h,\sigma)/\lpref(h,\sigma_{optimal})$.

\begin{corollary}
\label{cor:disag}
For any $V\subseteq U$ and preference function $h$ over $V$, the following bound holds:
\begin{equation}\label{eq:cor1}
\E_{s}[\lpref(\QS{s}{V}{h}, \sigma_{optimal})] \leq 2\, \lpref(h,\sigma_{optimal})\ .
\end{equation}
\end{corollary}
The corollary is immediate because technically any ranking and in particular $\sigma_{optimal}$ can be taken
as $\sigma^*$ in the proof of Theorem~\ref{th:approx3}.

\begin{corollary}
\label{cor:disag2}
Let $V \subseteq U$ be an arbitrary subset of $U$ and let
$\sigma_{optimal}$ be as above.  Then, the following bound holds for
the pairwise disagreement of the ranking $\QS{s}{V}{h}$ with respect to $h$:
\begin{equation}
\E_{s}[\lpref(h,\QS{s}{V}{h})] \leq 3\, \lpref(h, \sigma_{optimal}).
\end{equation}
\end{corollary}
\begin{proof}
This result follows directly Corollary~\ref{cor:disag} and the
application of a triangle inequality.  \qed
\end{proof}
The result in Corollary~\ref{cor:disag2} is known from previous work
\cite{DBLP:conf/stoc/AilonCN05,Ailon07}, where it is proven directly
without resorting to the intermediate inequality (\ref{eq:cor1}).  In
fact, a better bound of $2.5$ is known to be achievable using a more
complicated algorithm, which gives hope for a $1.5$ bound improving
Theorem~\ref{th:approx3}.

\subsection{Regret Bounds for Bipartite case}
\label{sec:regret}

We prove the second part (regret bounds) of Theorem~\ref{th:approx3bipartite}.  
By Observation~\ref{obs:regiia}, it is enough to prove that $\reg'_{rank}(A,D) \leq \reg'_{class}(h,D)$. Since in
the definition of $\reg'_{rank}$ and $\reg'_{class}$ the expectation over $V$ is outside the $\min$ operator, we may continue fixing $V$.  
Let $D_V$ denote the distribution over $\tau^*$ conditioned on $V$
It is now clearly enough to prove

\begin{eqnarray}\label{eq:regretdef2a}
\E_{D_V, s} [\lpref(Q^h_s, \tau^*)] -
\min_{\tilde \sigma} \E_{D_V}[\lpref(\tilde \sigma,\tau^*)] \label{eq:regprankdef2} \leq
\E_{D_V} [\lpref(h,\tau^*)] - \min_{\tilde h} \E_{D_V}[\lpref(\tilde h, \tau^*)]\label{eq:regpclassdef2bi}
\end{eqnarray}


We let $\mu(u,v)=\E_{D_V}[\tau^*(u,v)]$.  
(By pairwise IIA, $\mu(u,v)$ is the
same for all $V$ such that $u,v\in V$.)  By linearity of expectation, it suffices to show that
\begin{eqnarray}\label{eq:toprove}
\E_{s} [\lpref(Q^h_s, \mu)] -
\min_{\tilde \sigma} \lpref(\tilde \sigma,\mu)]  \leq
\lpref(h,\mu) - \min_{\tilde h}\lpref(\tilde h, \mu)\ .
\end{eqnarray}

Now let $\tilde \sigma$ and $\tilde h$ be the minimizers of the $\min$ operators on the left and right sides, respectively.  Recall
that for all $u,v\in V$, $\tilde h(u,v)$ can be taken greedily as a function of $\mu(u,v)$ and $\mu(v,u)$ (as in (\ref{eq:opth})).
\begin{equation}\label{eq:opth1}
\tilde h(u,v) = \begin{cases} 1 & \mu(u,v)> \mu(v,u) \\
0 & \mu(u,v) < \mu(v,u) \\
\1_{u>v} & \mbox{otherwise (equality)}\ .
\end{cases} 
\end{equation}

\noindent
Using Lemma~\ref{claim:decomp} and linearity, we write the LHS of (\ref{eq:toprove}) as:
\begin{equation*}\begin{split}
 & {n\choose 2}^{-1}\left ( \sum_{u<v} p_{uv}\alpha[h-\tilde \sigma,\mu]_{uv} +\sum_{u<v<w} p_{uvw}(\beta[\mu] - \gamma[\alpha[\tilde \sigma,\mu]])_{uvw}\right ) \\
\end{split}\end{equation*}

\noindent
and the RHS of (\ref{eq:toprove}) as:
\begin{equation*}\begin{split}
  {n\choose 2}^{-1}\left ( \sum_{u<v} p_{uv}\alpha[h-\tilde h,\mu]_{uv} +\sum_{u<v<w} p_{uvw}\gamma[\alpha[h-\tilde h, \mu]]_{uvw} \right )\ .
\end{split}\end{equation*}

Now, clearly for all $u,v$ by construction of $\tilde h$ we must have $\alpha[h-\tilde \sigma, \mu]_{uv} \leq \alpha[h-\tilde h, \mu]_{uv}$.
To conclude the proof of the theorem, we define 
$F: {n\choose 3} \mapsto \Rset$  as follows:


\begin{eqnarray}\label{eq:enoughtoprove5}
 F &=& \beta[\mu] - \gamma[\alpha[\tilde \sigma, \mu]] - (\gamma[\alpha[h,\mu]] - \gamma[\alpha[\tilde h, \mu]])\ .
\end{eqnarray}

It now suffices to prove that $F_{uvw} \leq 0$ for all $u,v,w\in V$.
Clearly $F$ is a  function of the values of 

\begin{equation}
\begin{split}
\mu(a,b)&:  a,b\in\{u,v,w\} \\
h(a,b)&: a,b\in\{u,v,w\} \\
\tilde \sigma(a,b)&: a,b \in \{u,v,w\}
\end{split}
\end{equation}
(recall that $\tilde h$ depends on $\mu$.) The $\mu$-variables can take values satisfying following constraints 
or all $u,v,w\in V$:
\begin{eqnarray}\label{eq:polytope}
 \mu(a,c) \leq \mu(a,b) + \mu(b,c) \ &\forall& \{a,b,c\}=\{u,v,w\}\label{eq:polytope1}\\
\mu(u,v) + \mu(v,w) + \mu(w,u) &=& \mu(v,u)+\mu(w,v)+\mu(u,w) \label{eq:polytope2}\\
 \mu(a,b)\geq 0 \ \  &\forall& a,b\in\{u,v,w\} \label{eq:polytope3}\ .
\end{eqnarray}
(the second constraint is obvious for any partition $\tau^*$.)

Let $P\subseteq \Rset^6$ denote the polytope defined by (\ref{eq:polytope1}-\ref{eq:polytope3}) in the variables $\mu(a,b)$ for $\{a,b\}\subseteq\{u,v,w\}$.
We subdivide $P$
into smaller subpolytopes on which the $\tilde h$ variables 
are constant.  
Up to symmetries, we can consider only two cases: (i) $\tilde h$ induces a cycle on $u,v,w$ and (ii) $\tilde h$ is cycle free on $u,v,w$.
\begin{enumerate}
\item[(i)]  Without loss of generality, assume $\tilde h(u,v) = \tilde h(v,w) = \tilde h(w,u) = 1$.  But this implies that $\mu(u,v) \geq \mu(v,u)$,
$\mu(v,w) \geq \mu(w,v)$ and $\mu(w,u) \geq \mu(u,w)$.  
Together with (\ref{eq:polytope2}) and (\ref{eq:polytope3}) this implies that
$\mu(u,v) = \mu(v,u)$, $\mu(v,w) = \mu(w,v)$ and $\mu(w,u) =\mu(u,w)$.   Consequently 
\begin{equation*}
\begin{split}
\beta[\mu]_{uvw} = \gamma[\alpha[\tilde \sigma, \mu]]_{uvw} &= \gamma[\alpha[h,\mu]]_{uvw} = \gamma[\alpha[\tilde h, \mu]]_{uvw} \\ &= \frac 1 3( {\mu(u,v) + \mu(v,w) + \mu(w,u)})\  
\end{split}
\end{equation*}
and 
$F_{uvw}=0$
, as required.  

\item[(ii)] Without loss of generality, assume $\tilde h(u,v) = \tilde h(v,w) = \tilde h(u,w) = 1$.   This implies that 
\begin{equation}\label{eq:subpolytope}
\begin{split}
\mu(u,v) &\geq \mu(v,u) \\
\mu(v,w) &\geq \mu(w,v) \\
\mu(u,w) &\geq \mu(w,u)\ . \\
\end{split}
\end{equation}
Let $\tilde P \subseteq P$ denote the polytope defined by (\ref{eq:subpolytope}) and (\ref{eq:polytope1})-(\ref{eq:polytope3}).  
Clearly $F$ is linear in the
$6$ $\mu$ variables when all the other variables are fixed.  Since $F$ is also homogenous in the $\mu$ variables, it
is enough to prove that $F\leq 0$  for $\mu$ taking values in $\tilde P'\subseteq \tilde P$, which is defined by adding the constraint, say, 
 $$\sum_{a,b\in\{u,v,w\}} \mu(a,b) = 2\ .$$
It is now enough to prove that $F\leq 0$ for $\tau^*$ being a vertex of of $\tilde P'$.  This finite set of cases can be easily checked to be:
\begin{equation*}
\begin{split}
(\mu(u,v), &\mu(v,u), \mu(u,w), \mu(w,u), \mu(w,v), \mu(v,w)) \in A\cup B \\
\\ \mbox{where } & A = \{ 
 (0,0,1,0,0,1),
 (1,0,1,0,0,0) \} \\
 &B = \{ (.5,.5,.5,.5,0,0), 
(.5,.5,0,0,.5,.5),
(0,0,.5,.5,.5,.5) \}\ .
\end{split}
\end{equation*}
The points in $B$ were already checked in case (i) (which is, geometrically, a boundary of case (ii)).  It remains to check the two points in $A$.
\begin{itemize}
\item case $(0,0,1,0,0,1)$:  Plugging in the definitions, one checks that:
\begin{equation*}\begin{split}
\beta[\mu]_{uvw} &= \frac 1 3 (h(w,v)h(v,u) + h(w,u)h(u,v)) \\
\gamma[\alpha[h,\mu]]_{uvw} &= \frac 1 3((h(u,v)h(v,w) + h(w,v)h(v,u))h(w,u) \\
                 &\ \ \ + (h(v,u)h(u,w) + h(w,u)h(u,v))h(w,v)) \\
\gamma[\alpha[\tilde h,\mu]]_{uvw} &= 0\ . \\
\end{split}\end{equation*}
Clearly $F$ could be positive only of $\beta_{uvw}=1$, which happens if and only if either $h(w,v)h(v,u)=1$ or $h(w,u)h(u,v)=1$.
In the former case we get that either $h(w,v)h(v,u)h(w,u)=1$ or $h(v,u)h(u,w)h(w,v)=1$, both implying $\gamma[\alpha[h,\mu]]_{uvw} \geq 1$, hence $F\leq 0$.
In the latter case either $h(w,u)h(u,v)h(w,v)=1$ or $h(u,v)h(v,w)h(w,u)=1$, both implying again $\gamma[\alpha[h,\mu]]_{uvw}\geq 1$ and hence $F\leq 0$.

\item case $(1,0,1,0,0,0)$:Plugging in the definitions, one checks that:
\begin{equation*}\begin{split}
\beta[\mu]_{uvw} &= \frac 1 3 (h(w,v)h(v,u) + h(v,w)h(w,u)) \\
\gamma[\alpha[h,\mu]]_{uvw} &= \frac 1 3((h(u,v)h(v,w) + h(w,v)h(v,u))h(w,u) \\
                 &\ \ \ + (h(u,w)h(w,v) + h(v,w)h(w,u))h(v,u))\ . \\
\gamma[\alpha[\tilde h,\mu]]_{uvw} &= 0\ .
\end{split}\end{equation*}
Now  $F$ could be positive if and only if either $h(w,v)h(v,u)=1$ or $h(v,w)h(w,u)=1$.
In the former case we get that either $h(w,v)h(v,u)h(w,u)=1$ or $h(v,u)h(u,w)h(w,v)=1$, both implying $\gamma[\alpha[h,\mu]]_{uvw} \geq 1$, hence $F\leq 0$.
In the latter case either $h(v,w)h(w,u)h(v,u)=1$ or $h(u,v)h(v,w)h(w,u)=1$, both implying again $\gamma[\alpha[h,\mu]]_{uvw}\geq 1$ and hence $F\leq 0$.

\end{itemize}

\end{enumerate}

\noindent This concludes the proof for the bipartite case.  
\qed

\subsection{Time Complexity}
\label{runtimesec}

Running QuickSort does not entail $\Omega(|V|^2)$ accesses to
$h_{u,v}$.  The following bound on the running time is proven in
Section~\ref{runtimesec}.

\begin{theorem}\label{runtime}
The expected number of times QuickSort accesses to the preference function $h$
is at most $O(n\log n)$.  Moreover, if only the top $k$ elements
are sought then the bound is reduced to $O(k\log k + n)$ by pruning
the recursion.
\end{theorem}

\def\outdeg{\operatorname{outdeg}} \def\indeg{\operatorname{indeg}} It
is well known that QuickSort on cycle free tournaments runs in time
$O(n\log n)$, where $n$ is the size of the set we want to sort.  That
it is true for QuickSort on general tournaments is a simple extension
(communicated by Heikki Mannila) which we present it here for self
containment.  The second part requires more work.  
\begin{proof}
Let $T(n)$ be the maximum expected running time of QuickSort on a
possibly cyclic tournament on $n$ vertices in terms of number of
comparisons. Let $G = (V, A)$ denote a tournament.  The main
observation is that each vertex $v \in V$ is assigned to the left
recursion with probability exactly $\outdeg(v)/n$ and to the right
with probability $\indeg(v)/n$, over the choice of the pivot.
Therefore, the expected size of both the left and right recursions is
exactly $(n - 1)/2$. The separation itself costs $n - 1$ comparisons.
The resulting recursion formula $T(n) \leq n - 1 + 2T((n - 1)/2)$
clearly solves to $T(n) = O(n\log n)$.

Assume now that only the $k$ first elements of the output are sought,
that is, we are interested in outputting only elements in positions
$1,\dots, k$.  The algorithm which we denote by $k$-QuickSort is
clear: recurse with $\min\set{k, n_L}$-QuickSort on the left side and
$\max\set{0, k-n_L-1}$-QuickSort on the right side, where $n_L, n_R$
are the sizes of the left and right recursions respectively and
$0$-QuickSort takes $0$ steps by assumption.  To make the analysis
simpler, we will assume that whenever $k\geq n/8$, $k$-QuickSort
simply returns the output of the standard QuickSort, which runs in
expected time $O(n \log n) = O(n + k\log k)$, within the sought bound.
Fix a tournament $G$ on $n$ vertices, and let $t_k(G)$ denote the
running time of $k$-QuickSort on $G$, where $k < n/8$.  Denote the
(random) left and right subtournaments by $G_L$ and $G_R$
respectively, and let $n_L = |G_L|, n_R = |G_R|$ denote their sizes in
terms of number of vertices.  Then, clearly,
\begin{equation}
\label{xxx} 
t_k(G) = n-1 + t_{\min\{k, n_L\}}(G_L) + t_{\max\{0, k-n_L-1\}}(G_R)\ .
\end{equation}
Assume by structural induction that for all $\{k',n':\ k'\leq n'< n\}$
and for all tournaments $G'$ on $n'$ vertices, $\E[t_{k'}(G')] \leq
cn' + c'k'\log k'$ for some global $c,c'>0$.  Then, by conditioning on
$G_L, G_R$, taking expectations on both sides of (\ref{xxx}) and by
induction,
\begin{equation*}\begin{split}
 \E[t_k(G)\ |\ G_L, &G_R] \leq \\ & n-1 + c n_L +
 c'\min\{k,n_L\}\log\min\{k, n_L\} +\\ & c n_R{\bf 1}_{n_L<k-1} +
 c'\max\{k-n_L-1,0\}\log \max\{k-n_L-1,0\}\ .
\end{split}\end{equation*}
By convexity of the function $x \mapsto x\log x$,
\begin{equation}
\small \min\{k,n_L\}\log\min\{k, n_L\} + \max\{k-n_L-1,0\}\log \max\{k-n_L-1,0\} \leq k\log k\ ,
\end{equation}
hence
\begin{equation}
\E[t_k(G)\ |\ G_L,G_R] \leq n-1 + c n_L + c n_R{\bf 1}_{n_L<k-1} +
c'k\log k.
\end{equation}
By conditional expectation,
\begin{equation*}
\E[t_k(G)] \leq n-1 + c(n-1)/2 + c'k\log
k + c\E[ n_R{\bf 1}_{n_L<k-1}].
\end{equation*}
To complete the inductive hypothesis, we need to bound $\E[ n_R{\bf
1}_{n_L<k-1}]$ which is bounded by $n\Pr[n_L<k-1]$. The event $\{n_L <
k-1\}$, equivalent to $\{ n_R > n-k\}$,occurs when a vertex of
out-degree at least $n - k \geq 7n/8$ is chosen as pivot.  For a random
pivot $v\in V$, where $V$ is the vertex set of $G$, $\E[\outdeg(v)^2]
\leq n^2/3 + n/2 \leq n^2/2.9$. Indeed, each pair of edges $(v, u_1)
\in A$ and $(v, u_2) \in A$ for $u_1 \neq u_2$ gives rise to a
triangle which is counted exactly twice in the cross-terms, hence
$n^2/3$ which upper-bounds $2 {n\choose 3}/n$; $n/2$ bounds the
diagonal). Thus, $\Pr[\outdeg(v) \geq 7n/8] = \Pr[\outdeg(v)^2 \geq
49n^2/64] \leq 0.46$ (by Markov).  Plugging in this value into our
last estimate yields
\begin{equation*}
\E[t_k(G)] \leq n-1 + c(n-1)/2 + c'k\log k + 0.46\times cn,
\end{equation*}
 which is at most $cn + c'k\log k$ for $c \geq 30$, as required. \qed
\end{proof}

\section{Discussion}
\label{sec:discussion}

\subsection{History of QuickSort}
The now standard textbook algorithm was
discovered by Hoare \cite{Hoare61} in 1961. Montague and Aslam
\cite{DBLP:conf/cikm/MontagueA02} experiment with QuickSort for
information retrieval by aggregating rankings from different sources
of retrieval.  They claim an $O(n\log n)$ time bound on the number of
comparisons although the proof seems to rely on the folklore QuickSort
proof without addressing the non-transitivity problem.  They prove
certain combinatorial bounds on the output of QuickSort and provide
empirical justification to its IR merits.  Ailon, Charikar and Newman
\cite{DBLP:conf/stoc/AilonCN05} also consider the rank aggregation
problem and prove theoretical cost bounds for many ranking problems on
weighted tournaments.  They strengthen these bounds by considering
nondeterministic pivoting rules (arising from solutions to certain
ranking LP's).  This work was extended by Ailon \cite{Ailon07} to deal
with rankings with ties (in particular, top-$k$ rankings).  Hedge et
al \cite{HedgeJWVZ07} and Williamson et al \cite{WilliamsonVZ07}
derandomize the random pivot selection step in QuickSort for many of
the combinatorial optimization problems studied by Ailon et al.

\subsection{The decomposition technique} 
The technique developed in Lemma~\ref{claim:decomp} is very general and can used for a
wide variety of loss functions and variants of QuickSort involving
nondeterministic ordering rules (see \cite{DBLP:conf/stoc/AilonCN05,Ailon07}).  Such results would typically amount to bounding
$\beta[X]_{uvw}/\gamma[Z]_{uvw}$ for some carefully chosen
functions $X,Z$ (depending on the application).


\subsection{Combinatorial Optimization vs. Learning} In
Ailon et al's work \cite{DBLP:conf/stoc/AilonCN05,Ailon07} the
QuickSort algorithm (sometimes referred to there as FAS-Pivot) is used
to approximate certain NP-Hard (see
\cite{DBLP:journals/siamdm/Alon06}) weighted instances of minimum
feedback arcset in tournaments.  There is much similarity between the
techniques used in the analyses, but there is also a significant
difference that should be noted.  In the minimum feedback arc-set
problem we are given a tournament $G$ and wish to find an acyclic
tournament $H$ on the same vertex set minimizing $\Delta(G,H)$, where
$\Delta$ counts the number of edges pointing in opposite directions
between $G,H$ (or a weighted version thereof).  
However, the cost we are
considering is $\Delta(G,H_\sigma)$ for some fixed acyclic tournament
$H_\sigma$ induced by some permutation $\sigma$ (the ground truth).  In this
work we showed in fact that if $G'$ is obtained from $G$ using
QuickSort, then $\E[\Delta(G',H_\sigma)] \leq 2\Delta(G, H_\sigma)$ for
\emph{any} $\sigma$ (from  Theorem~\ref{th:approx3}).  If $H$ is the optimal solution to the (weighted)
minimum feedback arc-set problem corresponding to $G$, then it is easy
to see that $\Delta(H,H_\sigma) \leq \Delta(G,H) + \Delta(G, H_\sigma)
\leq 2\Delta(G, H_\sigma)$.  
However, recovering $G$ is NP-Hard in
general.  Approximating $\Delta(G,H)$ (as done in the combinatorial
optimization world) by some constant factor\footnote{Kenyon-Mathiew
and Schudy \cite{KenyonMWS07} recently found such a PTAS for the
combinatorial optimization problem.}  $1+\varepsilon$ by an acyclic
tournament $H'$ only guarantees (using trivial arguments) a constant
factor of $2+\varepsilon$ as follows:
$$\Delta(H',H_\sigma) \leq \Delta(G,H') + \Delta(G, H_\sigma) \leq
(1+\varepsilon)\Delta(G,H) + \Delta(G, H_\sigma) \leq
(2+\varepsilon)\Delta(G, H_\sigma)\ .$$ This work therefore adds an
important contribution to \cite{DBLP:conf/stoc/AilonCN05,Ailon07,KenyonMWS07}.

\subsection{Normalization}
As mentioned earlier, the loss function $\lpref$ used in the bipartite case is
not exactly the same one used by Balcan et al in \cite{conf/colt/BalcanBBCLS07}.
There the total number of "misordered pairs" is divided not by ${n\choose 2}$ but rather  by the number of mixed pairs $u,v$ such that $\tau^*(u)\neq \tau^*(v)$
(see (\ref{normalization})).  We will not discuss the merits of each choice in this
work, but will show that the loss bound (first part) of Theorem~\ref{th:approx3bipartite} applies to their
normalization as well.
Indeed, let $\nu: \Pi(V) \rightarrow \Rset^+$ be any normalization function that depends on
a partition, and define a loss
$$ \lpref(X, \tau^*) = \nu(\tau^*)^{-1}\sum_{u\neq v}{X(u,v)\tau^*(v,u)}\ $$
for any $X:V\times V\rightarrow \Rset^+$ ($X$ can be a preference function $h$ 
or a ranking).  
In \cite{conf/colt/BalcanBBCLS07}, for example, $\nu(\tau^*)$ is taken as $|\{u,v: \tau^*(u) < \tau^*(v)\}|$
and here as ${n\choose 2}$.  
Since $V,\tau^*$ are fixed in the loss bound of Theorem~\ref{th:approx3bipartite}, this makes no difference for
the proof.
For the regret bound (second part) of Theorem~\ref{th:approx3bipartite} this however does not work.  Indeed, the pairwise IIA is
not enough to ensure that the event $u,v\in V$ determines $\nu(\tau^*)$, and we cannot simply swap $E_D$ and $\min_{\tilde h}$
as we did in Observation~\ref{obs:regiia}.  Working around this problem seems to require a stronger version
of IIA which does not seem natural.

\section{Lower Bounds}
\def\MFAT{\operatorname{MFAT}}
Balcan et al \cite{conf/colt/BalcanBBCLS07} prove a lower bound of a constant factor of $2$ for the regret bound of
the algorithm $\MFAT$,  defined as the solution to the minimum feedback arc-set problem on the tournament
$V$ with an edge $(u,v)$ if $h(u,v)=1$.   More precisely, they show an example of fixed $V, h$ and  $\tau^*\in \Pi(V)$ such 
that the classification regret of $h$ tends to $1/2$ of the ranking regret of $\MFAT$ on $V,h$.  Note that in this case, since
$\tau^*$ is fixed, the regret and loss are the same thing for both classification and ranking.  Here we show the following
stronger statement which is simpler to prove and applies in particular to the specific algorithm $\MFAT$ that is argued there.
\begin{theorem}
For any \emph{deterministic algorithm} $A$ taking input $V\subseteq U$ and preference function $h$ on $V$ and outputting
a ranking $\sigma \in S(V)$ there exists a distribution $D$ on $V,\tau^*$ such that
\begin{equation} 
\reg_{rank}(A, D) \geq 2\reg_{class}(h,D) 
\end{equation}
\end{theorem}

Note that this theorem says that in some sense, no deterministic algorithm that converts a preference function into a linear ranking
can do better than a randomized algorithm (on expectation) in the bipartite case.  Hence, randomization is essentially necessary in this
scenario.

The proof is by an adversarial argument.  In our construction, $D$ will always put all the mass on a single $V,\tau^*$ (deterministic input),
so the loss and regret are the same thing, and a similar argument will follow for the loss.  Also note that the normalization $\nu$ will
have no effect on the result.
\begin{proof}
Fix $V=\{u,v,w\}$, and $D$ puts all the weight on this particular $V$ and one partition $\tau^*$ (which we adversarially choose below). 
Assume $h(u,v)=h(v,w)=h(w,u)=1$ (a cycle).  Up to symmetry, there are two options for the output $\sigma$ of $A$ on $V,h$.
\begin{enumerate}
\item $\sigma(u) < \sigma(v) < \sigma(w)$.  In this case, the adversary chooses $\tau^*(w)=0$ and $\tau^*(u,v)=1$.  Clearly
$\reg_{class}(h,D)$ now equals $1/3$ ($h$ pays only for misordering $v,w$) but $\reg_{rank}(A,D)=2/3$ ($\sigma$ pays
for misordering the pairs $u,w$ and $v,w$).
\item $\sigma(w) < \sigma(v) < \sigma(u)$.  In this case, the adversary chooses $\tau^*(u)=0$ and $\tau^*(v,w)=1$.  Clearly
$\reg_{class}(h,D)$ now equals $1/3$ ($h$ pays only for misordering $u,w$) but $\reg_{rank}(A,D)=2/3$ ($\sigma$ pays
for misordering the pairs $u,v$ and $u,w$).
\end{enumerate}
This concludes the proof.
\end{proof}

\section{Conclusion}

We described a reduction of the learning problem of ranking to
classification. The efficiency of this reduction makes it practical
for large-scale information extraction and search engine
applications. A finer analysis of QuickSort is likely to further
improve our reduction bound by providing a concentration inequality
for the algorithm's deviation from its expected behavior using the
confidence scores output by the classifier. Our reduction leads to a
competitive ranking algorithm that can be viewed as an alternative to
the algorithms previously designed for the score-based setting.

\section{Acknowledgements}
\noindent
We thank John Langford and Alina Beygelzimer for helpful discussions.
\bibliographystyle{plain}
\bibliography{ml}
\end{document}